\documentclass[11pt,letterpaper]{article}
\usepackage{emnlp2016}
\usepackage{times}
\usepackage{latexsym}

 \emnlpfinalcopy

\title{The Structured Weighted Violations Perceptron Algorithm}

 \author{Rotem Dror \and Roi Reichart \\
 Faculty of Industrial Engineering and Management, Technion, IIT \\
        {\tt \{rtmdrr@campus$|$roiri@ie\}.technion.ac.il}}

\date{}

\usepackage{subfigure}
\usepackage{tikz}
\usepackage{tikz-dependency}
\usetikzlibrary{shapes,arrows}
\usetikzlibrary{intersections}
\usepackage[toc,page]{appendix}
\usepackage{amsfonts, amsmath, amssymb, amsthm}
\newtheorem{definition}{Definition}
\newtheorem{theorem}{Theorem}
\DeclareMathOperator*{\argmax}{arg\,max}

\usepackage{array}
\usepackage{multicol}
\usepackage{multirow}
\usepackage{algorithm}
\usepackage{algcompatible}
\newcommand{\isection}[2]{\section{#1}\label{ssec:#2}}
\newcommand{\isubsection}[2]{\subsection{#1}\label{ssec:#2}}
\newcommand{\secref}[1]{Section~\ref{ssec:#1}}
\newcommand{\my}[1]{}
\newcommand{\com}[1]{}

\begin{document}

\maketitle
%--------------------------------------------------------------------------%
\begin{abstract}
We present the \textit{Structured Weighted Violations Perceptron (SWVP)} algorithm, 
a new structured prediction algorithm 
that generalizes the Collins Structured Perceptron (CSP, \cite{CollinsPerceptron}).
Unlike CSP, the update rule of SWVP explicitly 
exploits the internal structure of the predicted labels. 
We prove the convergence of SWVP for linearly separable training sets, 
%SWVP converges to a weight 
%vector that separates the data, under certain conditions on the parameters of the algorithm. 
provide mistake and generalization bounds, 
%on: (a) the number of updates in the separable case; 
%(b) mistakes in the non-separable case; and (c) the probability to misclassify 
%an unseen example (generalization), 
and show that in the general case these bounds are tighter than those of the CSP special case.
In synthetic data experiments with data drawn from an HMM, various variants of SWVP 
substantially outperform its CSP special case. 
SWVP also provides encouraging initial dependency parsing results.
\end{abstract}

%--------------------------------------------------------------------------%
\isection{Introduction}{sec:intro}
%--------------------------------------------------------------------------%

The structured perceptron (\cite{CollinsPerceptron}, henceforth denoted CSP) 
is a prominent training algorithm for structured prediction models in NLP, due to its effective parameter 
estimation and simple implementation. It has been utilized in numerous NLP applications 
including word segmentation and POS tagging \cite{Zhang:08},
dependency parsing \cite{Koo:10,Goldberg:10,Martins:13}, 
semantic parsing \cite{Zettlemoyer:07} and information extraction \cite{Hoffmann:11,Reichart:12}, if to name just 
a few.

Like some training algorithms in structured prediction (e.g. structured SVM \cite{Taskar:04,SSVM}, 
MIRA \cite{Crammer:03} and LaSo \cite{Daume:05}), CSP considers in its update rule the difference 
between {\it complete} predicted and gold standard labels (Sec. 2). 
Unlike others (e.g. factored MIRA \cite{Mcdonald:05,Mcdonald:05b} and dual-loss based methods 
\cite{Meshi:10}) it does not exploit the structure of the predicted label. 
This may result in valuable information being lost. 

Consider, for example, the gold and predicted 
dependency trees of Figure~\ref{fig:dependencyExample}. The substantial difference between  
the trees may be mostly due to the difference in roots ({\it are} and {\it worse}, respectively). 
Parameter update w.r.t this mistake may thus be more useful than an update 
w.r.t the complete trees.
%Directly fixing this mistake may thus be in place instead of updating
%with respect to the full range of differences.
%have a substantial positive impact on the rest of the predicted tree. 

In this work we present a new perceptron algorithm with an update rule that exploits the structure of a predicted label 
when it differs from the gold label (\secref{sec:SWVP}). 
Our algorithm is called \textit{The Structured Weighted Violations Perceptron (SWVP)} 
as its update rule is based on a weighted sum of updates w.r.t {\it violating assignments} 
and {\it non-violating assignments}: assignments to the input example, derived from the predicted label,
that score higher (for violations) and lower (for non-violations)
than the gold standard label according to the current model. 
 %that score higher and lower, respectively, 

Our concept of {\it violating assignment} is based on \newcite{Huang:12} 
that presented a variant of the CSP algorithm where the argmax inference problem is replaced with a violation 
finding function. Their update rule, however, is identical to that of the CSP algorithm.
%we propose a novel update rule that 
%exploits the internal structure of the model's prediction regardless of the way this prediction was generated. 
Importantly, although CSP and the above variant do not exploit the internal 
structure of the predicted label, they are special cases of SWVP.

In \secref {sec:theory} we prove that for a linearly separable training set, 
SWVP converges to a linear separator of the data under certain conditions on the parameters of the algorithm, 
that are respected by the CSP special case.
%This is in line with the previous result stating that CSP converges 
%for any linearly separable training set \cite{CollinsPerceptron}.
We further prove mistake and generalization bounds for SWVP, and show that in the general case 
the SWVP bounds are tighter than the CSP's.

In \secref{sec:swvp-variants} we show that SWVP allows 
\textit{aggressive} updates, that exploit only violating assignments derived from the predicted label, 
and more \textit{balanced} updates, that exploit both violating and non-violating assignments.
In experiments with synthetic data generated by an HMM, we demonstrate that various SWVP variants substantially 
outperform CSP training. We also provide initial encouraging dependency parsing results, indicating the potential 
of SWVP for real world NLP applications.

\isection{The Collins Structured Perceptron}{sec:CSP}

In structured prediction the task is to find a mapping 
$f: \mathcal{X} \rightarrow \mathcal{Y}$, where $y \in \mathcal{Y}$ is a structured object rather than a scalar, 
and a feature mapping $\phi(x,y): \mathcal{X} \times \mathcal{Y}(x) \to\mathbb{R}^d$ is given.
%It is standard to assume an input feature representation $\phi(x,y): \mathcal{X} \times \mathcal{Y}(x) \to\mathbb{R}^d$.
In this work we denote $\mathcal{Y}(x) = \{y' | y' \in {D_Y}^{L_{x}} \}$, where $L_x$, a scalar, is the 
size of the allowed output sequence for an input $x$ and $D_Y$ is the domain of $y'_i$ for 
every $i \in \{1, \ldots L_x\}$.
\footnote{In the general case $L_x$ is a set of output sizes, which 
may be finite or infinite (as in constituency parsing \cite{Collins:97}).}  
%A representative structured problem with an unbounded 
%output space is phrase structure parsing \cite{Collins:97}.}  
Our results, however, hold for the general case of an output space with variable size vectors as well. 

%\isubsection{The Collins Structured Perceptron (CSP)}{sec:CSP}

The CSP algorithm (Algorithm~\ref{alg:structured_perceptron})
aims to learn a parameter (or weight) vector $\textbf{w}\in\mathbb{R}^d$, that separates the 
training data, i.e. for each training example $(x,y)$ it holds that: 
$y = \argmax_{y'\in \mathcal{Y}(x)}\textbf{w}\cdot\phi(x,y')$.
To find such a vector the algorithm iterates over the training set examples and solves the above inference ($argmax$) 
problem. If the inferred label $y^{*}$ differs from the gold label $y$ the 
update $\textbf{w} = \textbf{w} + \Delta\phi(x,y,y^{*})$ is performed. 
%(see algorithm~\ref{alg:structured_perceptron}). 
For linearly separable training data (see definition~\ref{def:separable}), CSP is proved to converge 
to a vector $\textbf{w}$ separating the training data.

\newcite{Collins-roark:04} and \newcite{Huang:12} expanded the CSP algorithm by proposing various alternatives to the
argmax {\it inference} problem which is often intractable in structured prediction problems (e.g. in high-order 
graph-based dependency parsing \cite{Mcdonald:06}). 
The basic idea is replacing the argmax problem with the search for  
a {\it violation}: an output label that the model scores higher than the gold standard label. The update rule 
in these CSP variants is, however, identical to the CSP's.
%that is scored higher by the model than the gold standard label.
We, in contrast, propose a novel update rule that 
exploits the internal structure of the model's prediction regardless of the way this prediction is generated. 
%In \secref{sec:theory} we show that the theoretical guarantees of SWVP hold for a rich family of 
%inference functions.

%In this paper, we extend this observation of \newcite{Huang:12} regarding parameter update with 
%and show that \textbf{w} can actually be updated with respect to linear combinations of mixed assignments, 
%under 

\begin{algorithm}[t!]
	\caption{The Structured Perceptron (CSP)}
	\label{alg:structured_perceptron}
\scriptsize
	\begin{algorithmic}[1]
		\STATEx {\bfseries Input:} data $D = \{x^i,y^i\}_{i=1}^{n}$, feature mapping $\phi$
		\STATEx {\bfseries Output:} parameter vector $\textbf{w}\in \mathbb{R}^d$
		\STATEx {\bfseries Define:} $\Delta\phi(x,y,z) \triangleq \phi(x,y) - \phi(x,z)$
		\STATE Initialize $\textbf{w}=0$.
		\REPEAT
		\FOR{each $(x^i,y^i)\in D$}
		\STATE $y^{*}=\argmax\limits_{y'\in \mathcal{Y}(x^i)} \textbf{w}\cdot\phi(x^i,y')$
		\IF{$y^{*} \neq y^i$} 
		\STATE $\textbf{w} = \textbf{w} + \Delta\phi(x^i,y^i,y^{*})$
		\ENDIF
		\ENDFOR
		\UNTIL{Convergence}
	\end{algorithmic}
\end{algorithm} 

%--------------------------------------------------------------------------%
\isection{The Structured Weighted Violations Perceptron (SWVP)}{sec:SWVP}
%--------------------------------------------------------------------------%

SWVP exploits the internal structure of a predicted label $y^{*} \neq y$ for a training example $(x,y) \in D$, by 
updating the weight vector with respect to sub-structures of $y^{*}$. We start by presenting the fundamental concepts 
at the basis of our algorithm.

\subsection{Basic Concepts}

\paragraph{Sub-structure Sets} We start with two fundamental definitions: 
{\bf (1)} An individual {\it sub-structure} of a 
structured object (or label) $y \in {D_Y}^{L_{x}}$, denoted with $J$, is defined to be a subset of 
indexes $J \subseteq [L_x]$;\footnote{We use the notation $[n] = \{1,2, \ldots n\}$.}
and {\bf (2)} A {\it set of substructures} for a training example $(x,y)$, 
denoted with $JJ_x$, is defined as $JJ_x \subseteq 2^{[L_x]}$.

\paragraph{Mixed Assignment} 
We next define the concept of a \textit{mixed assignment}:
\begin{definition} 
\label{MA}
For a training pair $(x,y)$ and a predicted label $y^{*} \in \mathcal{Y}(x)$, $y^{*} \neq y$, 
a \textbf{mixed assignment ($MA$)} vector denoted as $m^J(y^{*},y)$ is defined with respect to 
$J \in JJ_x$ as follows:
\begin{footnotesize}
$$m_k^J(y^{*},y)=\begin{cases}
y^{*}_k & k\in J \\
y_k & else
\end{cases}$$
\end{footnotesize}
\end{definition}
That is, a mixed assignment is a new label, derived from the predicted label $y^{*}$, that is identical to $y^{*}$ in 
all indexes in $J$ and to $y$ otherwise. For simplicity we denote $m^J(y^{*},y) = m^J$ when the 
reference $y^{*}$ and $y$ labels are clear from the context.

Consider, for example, the trees of Figure~\ref{fig:dependencyExample}, assuming that the top tree is $y$, 
the middle tree is $y^{*}$ and $J = [2,5]$.\footnote{We index the dependency tree words from 1 onwards.} 
In the $m^J(y^{*},y)$ (bottom) tree the heads of all the words are identical to those of the top tree, except 
for the heads of {\it mistakes} and of {\it then}. 
%The resulting tree is the bottom tree of the figure.
\my{RR: (1) the two roots problem; (2) the general problem.}

\paragraph{Violation} The next central concept is that of a violation, 
originally presented by \newcite{Huang:12}:
\begin{definition}
\label{def1}
A triple $(x,y,y^{*})$ is said to be a \textbf{violation} with 
respect to a training example $(x,y)$ and a parameter vector $\textbf{w}$ if for $y^{*} \in \mathcal{Y}(x)$ it 
holds that $y^{*}\neq y$ and $\textbf{w}\cdot\Delta\phi(x,y,y^{*}) \le 0$.
\end{definition}

The SWVP algorithm distinguishes between  $MA$s that are violations, and ones that are not.
For a triplet $(x,y,y^{*})$ and a set of substructures $JJ_x \subseteq 2^{[L_x]}$ 
we provide the following notations:
\begin{footnotesize}
\[ I(y^{*},y,JJ_x)^v \text{ } = \{J \in JJ_x|m^J \neq y, \textbf{w}\cdot\Delta\phi(x,y,m^J)\le 0 \} \]
\[I(y^{*},y,JJ_x)^{nv} = \{J\in JJ_x| m^J \neq y,  \textbf{w}\cdot\Delta\phi(x,y,m^J)> 0 \}\]
\end{footnotesize}
This notation divides the set of substructures into two subsets, one consisting of the 
substructures that yield violating MAs and one consisting of the substructures that yield non-violating MAs. 
Here again when the reference label $y^{*}$ and the set $JJ_x$ are known 
we denote: $I(y^{*},y,JJ_x)^v = I^v$, $I(y^{*},y,JJ_x)^{nv} = I^{nv}$ and $I=I^v\cup I^{nv}$.

%demonstrated how it is possible to update the parameter vector according 
%to a single set of edges that yields a violation, but this solution does not take into account multiple violations 
%that could be derived from a single predicted output. 

\paragraph{Weighted Violations} 

The key idea of SWVP is the exploitation of the internal structure of the predicted label in the update rule.
For this aim at each iteration we define the set of substructures, $JJ_{x}$, 
and then, for each $J \in JJ_{x}$, update the parameter vector, $\textbf{w}$, with respect to the 
mixed assignments, $MA^{J}$'s. This is a more flexible setup compared to CSP, as we can 
update with respect to the predicted output (if it is a violation, as is promised if inference is performed via argmax), 
if we wish to do so, as well as with respect to other mixed assignments.

Naturally, not all mixed assignments are equally important for the update rule. 
Hence, we weigh the different updates using a weight vector $\gamma$. 
%In \secref{sec:theory} we show what are the conditions on $\gamma$ under which SWVP converges 
%for linearly seprable training sets. In \secref{sec:swvp-variants} we discuss various selection strategies for 
%$\gamma$. 
This paper therefore extends the observation of \newcite{Huang:12} that perceptron parameter update
can be performed w.r.t violations (\secref{sec:CSP}), by showing that $\textbf{w}$ can actually 
be updated w.r.t linear combinations of mixed assignments, under certain conditions on the selected weights.

%\newcite{Huang:12} proved that if the parameter vector of CSP is updated with respect to a violation instead of 
%with respect to the solution of the argmax inference problem under the current model paramteres, 
%the algorithm still has the same convergence properties. 
%In this paper, we extend this observation and show that $\textbf{w}$ can actually be updated with respect to 
%linear combinations of mixed assignments, under certain conditions on the selected weights.
\begin{center}
	\includegraphics[width=0.5\textwidth]{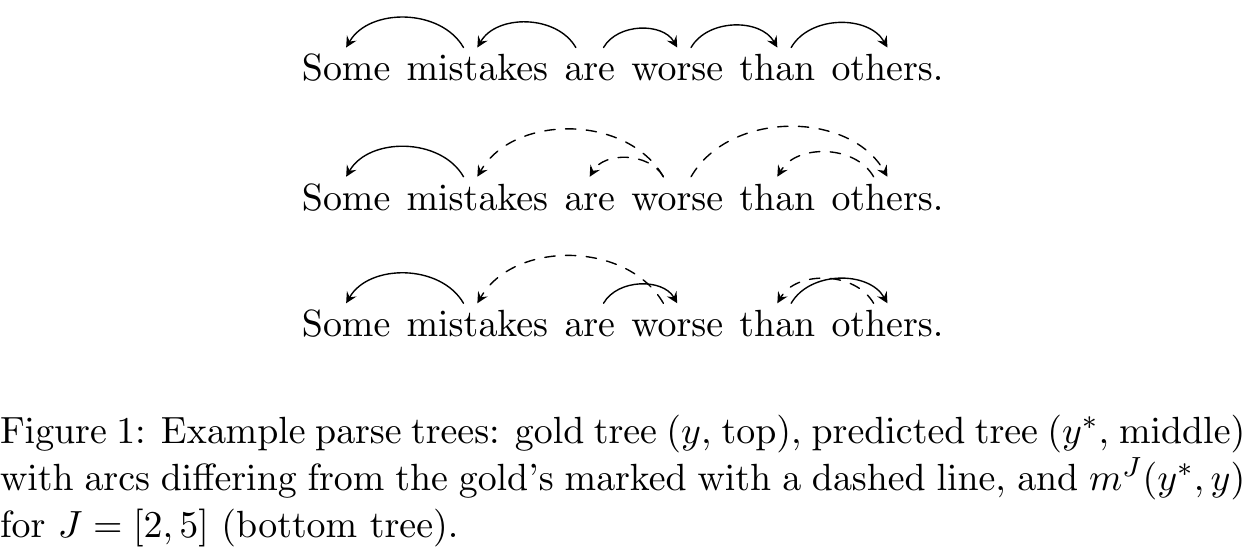}
	\label{fig:dependencyExample}
\end{center}

\subsection{Algorithm}

%\begin{figure}[t!]
%\centering
%\scriptsize
%%\resizebox{1.1\width}{.75\height}
%{\begin{dependency}[hide label, arc edge]
%\begin{deptext}
%Some \& mistakes \& are \& worse \& than \& others. \\
%\end{deptext}
%%\deproot[show label]{3}{root}
%\depedge{2}{1}{}
%\depedge{3}{2}{}
%\depedge{3}{4}{}
%\depedge{4}{5}{}
%\depedge{5}{6}{}
%\label{1}
%\end{dependency}}\\ 
%%\resizebox{1.1\width}{.75\height}
%{\begin{dependency}[hide label, arc edge]
%\begin{deptext}
%Some \& mistakes \& are \& worse \& than \& others. \\
%\end{deptext}
%%\deproot[show label, edge style={dashed,ultra thick}]{4}{root}
%\depedge{2}{1}{}
%\depedge[edge style={dashed}]{4}{2}{}
%\depedge[edge style={dashed}]{4}{3}{}
%\depedge[edge style={dashed}]{6}{5}{}
%\depedge[edge style={dashed}]{4}{6}{}
%\label{2}
%\end{dependency}}\\
%\begin{dependency}[hide label, arc edge]
%\begin{deptext}
%Some \& mistakes \& are \& worse \& than \& others. \\
%\end{deptext}
%\depedge{2}{1}{}
%\depedge[edge style={dashed}]{4}{2}{}
%\depedge{3}{4}{}
%\depedge[edge style={dashed}]{6}{5}{}
%\depedge{5}{6}{}
%\label{3}
%\end{dependency}
%\caption{Example parse trees: gold tree ($y$, top),
%predicted tree ($y^{*}$, middle) with arcs differing from the gold's 
%marked with a dashed line, and $m^J(y^{*},y)$ for $J = [2,5]$ (bottom tree).}
%%(upper: generated by the spaCy parser (https://spacy.io/demos/displacy); 
%%lower generated by the Stanford Online Parser (http://nlp.stanford.edu:8080/parser/)).}
%\label{fig:dependencyExample}
%\end{figure}

With these definitions we can present the SWVP algorithm (Algorithm \ref{alg:Wstructured_perceptron}). 
SWVP is in fact a family of algorithms differing with respect to two decisions that can be made 
at each pass over each training example $(x,y)$: the choice of the set $JJ_{x}$ and the implementation of the 
$\textsc{SetGamma}$ function. 

SWVP is very similar to CSP except for in the update rule. Like in CSP, the algorithm 
iterates over the training data examples and for each example it first predicts a label according to the current 
parameter vector $\textbf{w}$ (inference is discussed in \secref{sec:convProp}, property 2).
%\footnote{Although we employ the argmax function for inference (line 5), 
%in principle any label $y^{*}$ from which at least one violating MA can be derived is suitable 
%(see Definition 3). Like in other training algorithms, our basic reasoning is that argmax is a 
%good choice to base the parameter update on.} 
The main difference from CSP is in the update rule (lines 6-12). 
Here, for each substructure in the substructure set, $J \in JJ_{x}$, the algorithm generates a mixed assignment 
$m^J$ (lines 7-9). Then, $\textbf{w}$ is updated with a weighted sum of the mixed 
assignments (line 11), unlike in CSP where the update is held w.r.t the predicted assignment only.
%Then, the parameter vector $\textbf{w}$ is updated with a weighted sum of the updates that
%the standard CSP would have performed for an individual $m^J$ (line 11). 

The $\gamma(m^J)$ weights assigned to each of the $\Delta\phi(x,y,m^J)$ updates are defined 
by a \textsc{SetGamma} function (line 10). 
Intuitively, a $\gamma(m^J)$ weight should be higher the more the mixed assignment is assumed to convey 
useful information that can guide the update of $\textbf{w}$ in the right direction.
In \secref{sec:theory} we detail the conditions on \textsc{SetGamma} under which SWVP converges, and 
in \secref{sec:swvp-variants} we describe various \textsc{SetGamma} implementations.

Going back to the example of Figure~\ref{fig:dependencyExample}, 
%and suppose that the top and the bottom trees are gold and predicted trees respectively, 
one would assume (Sec. 1) that the head word prediction for  
{\it worse} is pivotal to the substantial difference between the two top trees (UAS of 0.2). 
CSP does not directly exploit this observation as it only updates its parameter 
vector with respect to the differences between complete assignments: $\textbf{w} = \textbf{w}+ \Delta\phi(x,y,z)$. 

In contrast, SWVP can exploit this observation in various ways. For example, it can generate a mixed 
assignment for each of the erroneous arcs where all other words are assigned their correct arc (according to the 
gold tree) except for that specific arc which is kept as in the bottom tree. Then, higher weights 
can be assigned to errors that seem more central than others. 
%SWVP can exploit this observation in various other ways.  
%Particularly, the set $JJ_{x^i}$ can include the complete predicted assignment (if $[L_x] \in JJ_{x^i}$)
%with or without additional mixed assignments. 
We elaborate on this in the next two sections.

%In the next section we prove the convergence of SWVP in the linearly separable case and 
%provide a mistake bound for non linearly separable data sets and a generalization bound.  

\begin{algorithm}[h]
	\caption{\footnotesize The Structured Weighted Violations Perceptron}
	\label{alg:Wstructured_perceptron}
\scriptsize
	\begin{algorithmic}[1]
		\STATEx {\bfseries Input:} data $D = \{x^i,y^i\}_{i=1}^{n}$, feature mapping $\phi$
%                 \STATE {\bfseries Input:} $JJ = \{JJ_{x^i} \subseteq 2^{[L_{x^i}]}  | (x^i,y^i) \in D\}$
		\STATEx {\bfseries Output:} parameter vector $\textbf{w}\in \mathbb{R}^d$
		\STATEx {\bfseries Define:} $\Delta\phi(x,y,z) \triangleq \phi(x,y) - \phi(x,z)$
		\STATE Initialize $\textbf{w}=0$.
		\REPEAT
		\FOR{each $(x^i,y^i)\in D$}
		\STATE $y^{*}=\argmax\limits_{y'\in \mathcal{Y}(x^i)}\textbf{w}\cdot\phi(x^i,y')$
		\IF{$y^{*} \neq y^i$} 
		\STATE {\bfseries Define:} $JJ_{x^i} \subseteq 2^{[L_{x^i}]}$
		\FOR{$J \in JJ_{x^i}$}
		\STATE {\bfseries Define:} $m^J$ s.t. $m_k^J=\begin{cases}
		y^{*}_k & k\in J \\
		y^i_k & else
		\end{cases}$
		\ENDFOR
		\STATE 	$\gamma = \textsc{SetGamma}()$
		\STATE $\textbf{w} = \textbf{w}+\sum\limits_{J \in I^v \cup I^{nv}}\gamma(m^J)\Delta\phi(x^i,y^i,m^J)$
%+\sum\limits_{J \in I^{nv} }\gamma(m^J)\Delta\phi(x^i,y^i,m^J)$
		\ENDIF
		\ENDFOR
		\UNTIL{Convergence}
	\end{algorithmic}
\end{algorithm}

%--------------------------------------------------------------------------%
\isection{Theory}{sec:theory}
%--------------------------------------------------------------------------%

%In this section we prove the convergence of the SWVP algorithm and describe its properties.
We start this section with the convergence conditions on the $\gamma$ vector which weighs  
the mixed assignment updates in the SWVP update rule (line 11). 
Then, using these conditions, we describe the relation between the SWVP and the CSP algorithms. 
After that, we prove the convergence of SWVP and analyse the derived properties of the algorithm.

\paragraph{$\gamma$ Selection Conditions}

Our main observation in this section is that SWVP converges under two conditions: 
\textbf{(a)} the training set $D$ is linearly separable; and \textbf{(b)} for any parameter vector $\textbf{w}$ achievable by the algorithm, 
there exists $(x,y) \in D$ with $JJ_x \subseteq 2^{[L_x]}$, such that for the predicted output 
$y^{*}\neq y$, \textsc{SetGamma} returns a $\gamma$ weight vector that respects the $\gamma$ 
selection conditions defined as follows:
\begin{definition}
\label{def:gamma}
	The \textbf{$\gamma$ selection conditions} for the SWVP algorithm are ($I=I^v\cup I^{nv}$):
\begin{footnotesize}
	\begin{align*}
	(1) & \sum_{J \in I} \gamma(m^J) = 1. \quad \gamma(m^J) \ge 0, \quad \forall J \in I.\\
	(2) & \quad \textbf{w}\cdot\sum\limits_{J\in I}\gamma(m^J)\Delta\phi(x^i,y^i,m^J) \leq 0.
	\end{align*}
\end{footnotesize}
\end{definition}

%Both conditions are necessary to assure convergence of the algorithm. Condition (1) bounds the weights and 
%condition (2) is an equivalent condition to the violation condition only in the weighted updates scenario. 
\com{Note that condition (2) is equivalent to:
%{\begin{footnotesize}
\begin{multline}
\footnotesize
|\textbf{w}\cdot\sum\limits_{J \in I^{nv}}\gamma(m^J)\Delta\phi(x,y,m^J)| \nonumber
\le \\ |\textbf{w}\cdot\sum\limits_{J \in I^{v}}\gamma(m^J)\Delta\phi(x,y,m^J)|
\end{multline}
%\end{footnotesize}}
since $\textbf{w} \cdot \sum_{J \in I^{nv}} \gamma(m^J) \Delta \phi(x,y,m^J) \ge 0$ and \\
$\textbf{w} \cdot \sum_{J \in I^{v}} \gamma(m^J) \Delta \phi(x,y,m^J) \le 0$.
}
With this definition we are ready to prove the following property.

\paragraph{SWVP Generalizes the CSP Algorithm}
 
We now show that the CSP algorithm is a special case of SWVP. 
CSP can be derived from SWVP when taking: 
$JJ_{x} = \{[L_x]\}$, and $\gamma (m^{[L_{x}]}) = 1$ for every $(x,y) \in D$.
With these parameters, the $\gamma$ selection conditions hold for every $\textbf{w}$ and $y^{*}$. 
Condition (1) holds trivially as there is only one $\gamma$ coefficient and it is equal to 1. 
Condition (2) holds as $y^{*} = m^{[L_x]}$ and hence $I = \{[L_x]\}$ and  
$w\cdot \sum\limits_{J\in I} \Delta\phi(x,y,m^J) \le 0$.
%$|\textbf{w} \cdot \Delta \phi(x,y,m^J)| \ge 0$.

%The violation fixing perceptron variants of \cite{Huang:12} differ from CSP only in the inference function, 
%replacing the $argmax$ with a \textit{findViolation} function that does not perform full inference. Hence, 
%these algorithms 
%are also special cases of SWVP with a single substructure $JJ_x \in \{[L_x]\}$ that is a violation 
%and $\gamma (m^J) = 1$ for all $(x,y)\in D$.

\isubsection {Convergence for Linearly Separable Data}{sec:convergence}
Here we give the theorem regarding the convergence of the SWVP in the separable case. We first define:
\begin{definition}
\label{def:separable}
A data set $D=\{x^i,y^i\}_{i=1}^n$ is \textbf{linearly separable with margin $\delta>0$} if there exists some 
vector $\textbf{u}$ with $\|\textbf{u}\|_2=1$ such 
that for all $i$:
\begin{footnotesize}
\[\textbf{u}\cdot\Delta\phi(x^i,y^i,z) \ge \delta,\forall z \in \mathcal{Y}(x^i).\]
\end{footnotesize}
\end{definition}
\begin{definition}
The \textbf{radius} of a data set $D=\{x^i,y^i\}_{i=1}^n$ is the minimal scalar $R$ s.t for all $i$: 
\begin{footnotesize}
\[\|\Delta\phi(x^i,y^i,z)\|\le R, \forall z \in \mathcal{Y}(x^i).\]
\end{footnotesize}
\end{definition}

We next extend these definitions:
\begin{definition}
Given a data set $D=\{x^i,y^i\}_{i=1}^n$ and a set $JJ = \{JJ_{x^i} \subseteq 2^{[L_{x^i}]}  | (x^i,y^i) \in D\}$, 
D is \textbf{linearly separable w.r.t $JJ$, with margin $\delta^{JJ}>0$} if there exists a vector $\textbf{u}$ 
with $\|\textbf{u}\|_2=1$ such that: $\textbf{u}\cdot\Delta\phi(x^i,y^i,m^J(z,y^i)) \ge \delta^{JJ}$ for 
all $i,z \in \mathcal{Y}(x^i), J \in JJ_{x^i}$. %and for every mixed assignment $m^J(z,y^i)$. 
\end{definition}
%Likewise, we define: 
\begin{definition}
The \textbf{mixed assignment radius w.r.t $JJ$} of a data set $D=\{x^i,y^i\}_{i=1}^n$ is a 
constant $R^{JJ}$ s.t for all $i$ it holds that: 
\begin{footnotesize}
\[\|\Delta\phi(x^i,y^i,m^J(z,y^i))\|\le R^{JJ}, \forall z \in \mathcal{Y}(x^i), J \in JJ_{x^i}.\]
\end{footnotesize}
\end{definition}

With these definitions we can make the following observation (proof in \ref{appendix:Observation1}): \\
{\bf Observation 1.} For linearly separable data $D$ and a set $JJ$, every unit vector $\textbf{u}$ that separates 
the data with margin $\delta$, also separates the data with respect to mixed assignments with $JJ$, 
with margin $\delta^{JJ} \ge \delta$. Likewise, it holds that $R^{JJ} \leq R$.

We can now state our convergence theorem. While the proof of this theorem resembles that of the CSP \cite{CollinsPerceptron}, unlike the CSP proof 
the SWVP proof relies on the {\it $\gamma$ selection conditions} presented above and on the {\it Jensen inequality}.

\begin{theorem}
\label{theorem:separable}
For any dataset D, linearly separable with respect to $JJ$ with margin $\delta^{JJ}>0$, 
the SWVP algorithm terminates after $t\le \frac{(R^{JJ})^2}{(\delta^{JJ})^2}$ steps, where $R^{JJ}$ is the mixed assignment 
radius of D w.r.t. $JJ$.
\end{theorem}

\begin{proof}
Let $\textbf{w}^t$ be the weight vector before the $t^{th}$ update, 
thus $\textbf{w}^1=0$. Suppose the $t^{th}$ update occurs on 
example $(x,y)$, i.e. for the predicted output $y^{*}$ it holds that $y^{*} \neq y$.
We will bound $\|\textbf{w}^{t+1}\|^2$ from both sides.\\
First, it follows from the update rule of the algorithm that:
%\begin{align*}
%\scriptsize
$\textbf{w}^{t+1} = \enskip \textbf{w}^t+\sum\limits_{J\in I^v \cup I^{nv}}\gamma(m^J)\Delta\phi(x,y,m^J).$
%\end{align*}
For simplicity, in this proof we will use the notation $I^v \cup I^{nv} = I$. 
Hence, multiplying each side of the equation by $\textbf{u}$ yields:
\begin{small}
\begin{align*}
\textbf{u}\cdot\textbf{w}^{t+1} = &\enskip \textbf{u}\cdot\textbf{w}^t+\textbf{u}\cdot\sum\limits_{J\in I}\gamma(m^J)\Delta\phi(x,y,m^J)\\
=&\enskip \textbf{u}\cdot\textbf{w}^t+\sum\limits_{J\in I}\gamma(m^J)\textbf{u}\cdot\Delta\phi(x,y,m^J)\\
\ge &\enskip \textbf{u}\cdot\textbf{w}^t+\sum\limits_{J\in I}\gamma(m^J)\delta^{JJ} \enskip \text{(margin property)}\\
\ge&\enskip \textbf{u}\cdot\textbf{w}^t+\delta^{JJ} 
\ge \ldots\ge t\delta^{JJ}.
\end{align*}
\end{small}
The last inequality holds because $\sum_{J\in I}\gamma(m^J)=1$.
From this we get that $\|w^{t+1}\|^2 \ge (\delta^{JJ})^2 t^2$ since $\|u\|$=1.
Second,
\begin{small}
\begin{align*}
\|\textbf{w}^{t+1}\|^2 = &\enskip \|\textbf{w}^t+\sum\limits_{J\in I}\gamma(m^J)\Delta\phi(x,y,m^J)\|^2\\  
=&\enskip \|\textbf{w}^t\|^2 + \|\sum\limits_{J\in I}\gamma(m^J)\Delta\Phi(x,y,m^J) \|^2\\ 
&+2\textbf{w}^t\cdot\sum\limits_{J\in I}\gamma(m^J)\Delta\Phi(x,y,m^J).
\end{align*}
\end{small}
From $\gamma$ selection condition (2) we get that:
\begin{small}
\begin{align*}
\|\textbf{w}^{t+1}\|^2 \le&\enskip \|\textbf{w}^t\|^2 + \|\sum\limits_{J \in I}\gamma(m^J)\Delta\Phi(x,y,m^J) \|^2\\
\le &\enskip \|\textbf{w}^t\|^2 + \sum\limits_{J \in I}\gamma(m^J)\|\Delta\Phi(x,y,m^J) \|^2\\
\le&\enskip  \|\textbf{w}^t\|^2 + (R^{JJ})^2. \enskip \text{(radius property)}
\end{align*}
\end{small}
The inequality one before the last results from the {\it Jensen inequality} 
which holds due to {\bf (a)} $\gamma$ selection condition (1); and {\bf (b)} the squared norm function being convex.
From this we finally get:
\begin{footnotesize}
\begin{align*}
\|\textbf{w}^{t+1}\|^2 \le& \enskip\|\textbf{w}^t\|^2 + (R^{JJ})^2 \le \ldots\le t(R^{JJ})^2.
\end{align*}
\end{footnotesize}
Combining the two steps we get: 
\begin{footnotesize}
$$(\delta^{JJ})^2t^2\le\| \textbf{w}^{t+1}\|^2 \le t(R^{JJ})^2.$$
\end{footnotesize}
From this it is easy to derive the upper bound in the theorem:
\begin{footnotesize}
$t \le \frac{(R^{JJ})^2}{(\delta^{JJ})^2}$ .
\end{footnotesize}
\end{proof}

\isubsection{Convergence Properties}{sec:convProp}

We next point on three properties of the SWVP algorithm, derived from its convergence proof:

{\bf Property 1 (tighter iterations bound)} The convergence proof of CSP \cite{CollinsPerceptron}
is given for a vector $\textbf{u}$ that linearly separates the data, with margin $\delta$ and for a data radius $R$. 
Following observation 1, it holds that in our case, $\textbf{u}$ also linearly separates the data with respect to 
mixed assignments with a set $JJ$ and with margin $\delta^{JJ} \ge \delta$. 
Together with the definition of $R^{JJ} \le R$ we get that: $\frac{({R^{JJ}})^2}{{(\delta^{JJ}})^2} \leq \frac{R^2}{\delta^2}$.
This means that the bound on the number of updates made by SWVP is tighter than the bound of CSP. 

\com{
{\bf Property 2 (convergence under conditions)} 

%As stated above, the convergence proof holds under two conditions: 
%(a) the training data $D$ is linearly separable; and (b) for any weight vector $\textbf{w}$ achievable by the algorithm, 
%there exists $(x,y) \in D$ with $JJ_x \subseteq 2^{[L_x]}$, such that for a predicted output $z \neq y$, \textsc{SetGamma} 
%returns a $\gamma$ parameter vector that respects the $\gamma$ selection conditions. 

The $\gamma$ selection conditions paragraph states two conditions ((a) and (b)) under which the 
convergence proof holds. If condition (b) does not hold then SWVP gets stuck when it reaches 
the vector $\textbf{w}$ for which the condition does not hold, as it cannot perform any updates. 
In such a case the algorithm will not converge to a weight vector $\textbf{w}$ that separates the data.
Dynamic updates (property 4 below), may be the remedy for this.
}

{\bf Property 2 (inference)} 
%Property 2 characterizes the inference problem in SWVP: finding an assignment that is 
%both desirable given the current model, and respects the $\gamma$ selection conditions. 
From the $\gamma$ selection conditions it holds that any label from which at least one 
violating MA can be derived through $JJ_x$ is suitable for an update. This is because in such a case 
we can choose, for example, a \textsc{SetGamma} function that assigns the weight of 1 to that MA, 
and the weight of 0 to all other MAs.

Algorithm \ref{alg:Wstructured_perceptron} employs the $argmax$ inference function, following 
the basic reasoning that it is a good choice to base the parameter update on. 
Importantly, if the inference function is argmax and the algorithm performs an update ($y^{*} \neq y$), 
this means that $y^{*}$, the output of the argmax function, is a violating MA by definition. 
%(as long as $\{[L_x]\} \in JJ_{x}$).
However, it is obvious that solving the inference problem and the optimal $\gamma$ assignment problems jointly 
may result in more informed parameter ($\textbf{w}$) updates.
We leave a deeper investigation of this issue to future research.
%which can then be complemented by various \textsc{SetGamma} functions (e.g. one that gives the weight of 1 
%to the mixed assignment with the highest violation and 0 to the rest). 

{\bf Property 3 (dynamic updates)} The $\gamma$ selection conditions paragraph states two conditions ((a) and (b)) 
under which the convergence proof holds. While it is trivial for \textsc{SetGamma} to generate a $\gamma$ vector 
that respects condition (a), if there is a parameter vector \textbf{w'} achievable by the algorithm 
for which \textsc{SetGamma} cannot generate $\gamma$ that respects condition (b),
SWVP gets stuck when reaching \textbf{w'}. 

This problem can be solved with {\it dynamic updates}. A deep look into the convergence proof reveals that 
the set $JJ_x$ and the \textsc{SetGamma} function can actually differ between iterations. 
While this will change the bound 
on the number of iterations, it will not change the fact that the algorithm converges if the data is linearly 
separable. This makes SWVP highly flexible as it can always back off to the CSP setup of $JJ_x = \{[L_x]\}$, 
and $\forall (x,y) \in D: \gamma (m^{[L_x]}) = 1$, update its parameters and continue with its original $JJ$ 
and \textsc{SetGamma} when this option becomes feasible. If this does not happen, the algorithm can continue 
till convergence with the CSP setup.

\subsection {Mistake and Generalization Bounds}
The following bounds are proved:
the number of updates in the separable case (see Theorem~\ref{theorem:separable}); the number of mistakes in the non-separable case (see Appendix~\ref{appendix:mistake}); and the probability to misclassify an unseen example (see Appendix~\ref{appendix:mistake}).
%The supplementary material states and proves generalization and mistake bounds for SWVP
%on: (a) the number of updates in the separable case; 
%(b) the number of mistakes in the non-separable case; and (c) the probability to misclassify 
%an unseen example (generalization). 
It can be shown that in the general case these bounds are tighter than those of the CSP special case. We next discuss variants of SWVP.

%--------------------------------------------------------------------------%
\isection{Passive Aggressive SWVP}{sec:swvp-variants}
%--------------------------------------------------------------------------%
Here we present types of update rules that can be implemented within SWVP. 
Such rule types are defined by: (a) the selection of $\gamma$, which should respect 
the \textit{$\gamma$ selection conditions} (see Definition~\ref{def:gamma}) and 
(b) the selection of  $JJ = \{JJ_{x} \subseteq 2^{[L_{x}]}  | (x,y) \in D\}$, 
the substructure sets for the training examples.

%\isubsection{$\gamma$ Selection}{sec:gamma-selection}
\paragraph{$\gamma$ Selection}
A first approach we consider is the \textit {aggressive approach}\footnote{We borrow the term \textit{passive-aggressive} 
from \cite{Crammer:06}, despite the substantial difference between the works.}
where only mixed assignments that are violations $\{m^J:J \in I^{v}\}$ are exploited (i.e. for all $J \in I^{nv}, \gamma (m^J) = 0$). 
Note, that in this case condition (2) of the \textit{$\gamma$ selection conditions} trivially holds as: 
$\textbf{w} \cdot \sum\limits_{J \in I^{v}} \gamma(m^J) \Delta \phi(x,y,m^J) \leq 0$. 
The only remaining requirement is that condition (1) also holds, i.e. 
that $\sum_{J \in I^v} \gamma (m^J) = 1.$
% In this approach, for a training example $(x,y)$ and a predicted label $y^{*} \neq y$ the update rule is: $\textbf{w} = \textbf{w} + \sum\limits_{J \in I^v}\gamma(m^J(y^{*},y))\Delta\phi(x,y,m^J(y^{*},y))$
% for $\sum\limits_{J \in I^v} \gamma (m^J) = 1$. 

The opposite, \textit{passive approach}, exploits only non-violating MA's $\{m^J:J \in I^{nv}\}$.
However, such $\gamma$ assignments do not respect $\gamma$ selection condition (2), as they yield:
$\textbf{w} \cdot \sum_{J \in I^{nv}} \gamma(m^J) \Delta \phi(x,y,m^J) \le 0$ which holds if and only if for every $J \in I^{nv}$, $\gamma (m^J) = 0$ that in turn contradicts condition (1).

Finally, we can take a \textit{balanced approach} which gives a positive $\gamma$ coefficient for 
at least one violating MA and at least one positive $\gamma$ coefficient for a non-violating MA. 
This approach is allowed by SWVP as long as both $\gamma$ selection conditions hold.

We implemented two weighting methods, both based on the concept of margin:\\
\textbf{(1) Weighted Margin (WM)}:
$\gamma(m^J) = \frac{|\textbf{w}\cdot\Delta\phi(x,y,m^J)|^{\beta}}{\sum\limits_{J' \in JJ_x}|\textbf{w}\cdot\Delta\phi(x,y,m^{J'})|^{\beta}}$ \\
\textbf{(2) Weighted Margin Rank (WMR)}: $\gamma(m^J) = \left(\frac{|JJ_x|-r}{|JJ_x|}\right)^{\beta}$.
where $r$ is the rank of $|\textbf{w}\cdot\Delta\phi(x,y,m^J(y^{*},y))|$ among 
the $|\textbf{w}\cdot\Delta\phi(x,y,m^{J'}(y^{*},y))|$ values for $J' \in JJ_x$.
%defined as follows:\\
%\textbf{(1)} Weighted Margin (WM):
%\begin{footnotesize}
%\begin{equation}
%\gamma(m^J) = \frac{|\textbf{w}\cdot\Delta\phi(x,y,m^J)|^{\beta}}{\sum\limits_{J' \in JJ_x}|\textbf{w}\cdot\Delta\phi(x,y,m^{J'})|^{\beta}} \nonumber
%\end{equation}
%\end{footnotesize}
%\textbf{(2)} Weighted Margin Rank (WMR):
%\begin{scriptsize}
%\begin{equation}
%\gamma(m^J) = \left(\frac{|JJ_x|-r}{|JJ_x|}\right)^{\beta} \nonumber
%\end{equation}
%\end{scriptsize}

%\begin{enumerate}
%\item[\textbf{(1)}] \textbf{ Weighted Margin (WM)}:
%\[\gamma(m^J) = \frac{|\textbf{w}\cdot\Delta\phi(x,y,m^J)|^{\beta}}{\sum\limits_{J' \in JJ_x}|\textbf{w}\cdot\Delta\phi(x,y,m^{J'})|^{\beta}}\]
%\item[\textbf{(2)}] \textbf{ Weighted Margin Rank (WMR)}:
%\[\gamma(m^J) = \left(\frac{|JJ_x|-r}{|JJ_x|}\right)^{\beta}\] 
%where $r$ is the rank of $\textbf{w}\cdot\Delta\phi(x,y,m^J(y^{*},y))$ among 
%the $\textbf{w}\cdot\Delta\phi(x,y,m^{J'}(y^{*},y))$ values for $J' \in JJ_x$.
%\end{enumerate}

Both schemes were implemented twice, within a balanced approach (denoted as B) and an aggressive approach (denoted as A).\footnote{
For the aggressive approach the equations for schemes (1) and (2) are changed such that $JJ_x$ is replaced with $I(y^*,y,JJ_x)^v$.}
The aggressive schemes respect both $\gamma$ selection conditions. 
The balanced schemes, however, respect the first condition but not necessarily the second. 
Since all models that employ the balanced weighting schemes converged after at most 10 iterations, 
we did not impose this condition (which we could do by, e.g., excluding terms for $J \in I^{nv}$ till condition (2) holds).

%In \secref{sec:experiments} we describe our implementation of the aggressive 
%and the balanced $\gamma$ selection approaches in this paper.

%\isubsection{$JJ$ Selection}{sec:JJ-selection}
\paragraph{$JJ$ Selection}
Another choice that strongly affects the updates made by SWVP is that of $JJ$. 
A choice of $JJ_{x} = 2^{[L_{x}]},$  for every $(x,y) \in D$ results in an update 
rule which considers all possible mixing assignments derived from the predicted label $y^{*}$ and the gold label $y$. 
Such an update rule, however, requires computing a sum over an exponential number of terms ($2^{L_{x}}$) and is therefore highly inefficient. 

Among the wide range of alternative approaches, in this paper we exploit \textit{single difference} mixed assignments. 
In this approach we define: $JJ = \{JJ_{x} = \{\{1\}, \{2\}, \ldots \{L_{x}\}\} | (x,y) \in D\}$. 
%That is, the derived mixed assignment will be defined as follows. 
For a training pair $(x,y) \in D$, a predicted label $y^{*}$ and $J = \{j\} \in JJ_{x}$, we will have: \\ 
$m_k^J(y^{*},y)=\begin{cases}
y_k & k\neq j \\
y^{*}_k & k=j
\end{cases}$

Under this approach for the pair $(x,y) \in D$ only $L_{x}$ terms are summed in the SWVP update rule. 
%We further experiment with $JJ_x$ of size 3 and 5, meaning we consider all threesomes and quintets of index combinations, 
%the results can be seen in \secref{sec:results}. 
We leave a further investigation of $JJ$ selection approaches to future research.

%--------------------------------------------------------------------------%
\isection{Experiments}{sec:experiments}
%--------------------------------------------------------------------------%

%We compare the SWVP to the CSP algorithm in both synthetic and real data experiments.

%\isubsection{Synthetic Experiments}{sec:synthetic}
%\paragraph{Data}

\paragraph{Synthetic Data}

We experiment with synthetic data generated by a linear-chain, first-order Hidden Markov Model (HMM, \cite{Rabiner:86}).
Our learning algorithm is a liner-chain conditional random field (CRF, \cite{Lafferty:01}): 
$P(y|x) = \frac{1}{Z(x)} \prod_{i=1:L_x} exp(w \cdot \phi (y_{i-1},y_i,x))$ (where $Z(x)$ is 
a normalization factor) with binary indicator features 
$\{x_i, y_i, y_{i-1}, (x_i,y_i), (y_i,y_{i-1}), (x_i, y_i, y_{i-1})\}$ for the triplet $(y_i, y_{i-1},x)$.

A dataset is generated by iteratively sampling $K$ items, each is sampled as follows.
We first sample a hidden state, $y_1$, from a uniform prior distribution. Then, iteratively,  
for $i = 1,2, \ldots, L_x$ we sample an observed state from the emission probability and (for $i < L_x$) 
a hidden state from the transition probability.
We experimented in 3 setups. In each setup we generated 10 datasets that were subsequently divided 
to a 7000 items training set, a 2000 items development set and a 1000 items test set. In all datasets, for each 
item, we set $L_x = 8$. We experiment in three conditions: 
(1) simple(++), learnable(+++), (2) simple(++), learnable(++) and (3) simple(+), learnable(+).\footnote{
Denoting $D_x = [C_x] $, $D_y = [C_y] $,  
and a permutation of a vector $v$ with $perm(v)$, 
%and the entire probability with $P(y' \in {D_y} | y)$ and $P(x \in {D_x} | y)$ , 
the parameters of the different setups are:
(1) simple(++), learnable(+++): $C_x = 5$, $C_y = 3$, $P(y'|y) = perm(0.7,0.2,0.1)$, $P(x|y) = perm(0.75, 0.1, 0.05, 0.05, 0.05)$.
(2) simple(++), learnable(++): $C_x = 5$, $C_y = 3$, $P(y'|y) = perm(0.5,0.3,0.2)$, $P(x|y) = perm(0.6, 0.15, 0.1, 0.1, 0.05)$.
(3) simple(+), learnable(+): $C_x = 20$ , $C_y = 7$ , $P(y'|y) = perm(0.7,0.2,0.1,0, \dots,0))$, $P(x|y) = perm(0.4, 0.2, 0.1, 0.1, 0.1,0, \ldots, 0)$.	 
}

%\begin{small}
%\begin{enumerate}
%	\item[(1)] \textbf{simple(++), learnable(+++)}: $C_x = 5$, $C_y = 3$, $P(y'|y) = perm(0.7,0.2,0.1)$, $P(x|y) = perm(0.75, 0.1, 0.05, 0.05, 0.05)$.
%	\item[(2)] \textbf{simple(++), learnable(++)}: $C_x = 5$, $C_y = 3$, $P(y'|y) = perm(0.5,0.3,0.2)$, $P(x|y) = perm(0.6, 0.15, 0.1, 0.1, 0.05)$.
%	\item[(3)] \textbf{simple(+), learnable(+)}: $C_x = 20$ , $C_y = 7$ , $P(y'|y) = perm(0.7,0.2,0.1,0, \dots,0))$, $P(x|y) = perm(0.4, 0.2, 0.1, 0.1, 0.1,0, \ldots, 0)$.	 
%\end{enumerate}
%\end{small}

%\paragraph{Evaluation}
%Recall that we are experimenting in three setups, each with 10 instance datasets. 
For each dataset (3 setups, 10 datasets per setup) we train variants of the SWVP algorithm differing in the $\gamma$ 
selection strategy (WM or WMR, \secref{sec:swvp-variants}), being aggressive (A) or passive (B), and in their $\beta$ 
parameter ($\beta = \{0.5, 1, \ldots, 5\}$). Training is done on the training subset and the best 
performing variant on the development subset is applied to the test subset.
For CSP no development set is employed as there is no hyper-parameter to tune. 
%for each dataset we only need to train the algorithm on the training set and then apply 
%it to the test set as there is no hyper-parameter to tune. 
We report averaged accuracy (fraction of observed states for which the model successfully predicts the hidden state value) 
across the test sets, together with the standard deviation. 

%\isubsection{Dependency Parsing}{sec:DEP} 
\paragraph{Dependency Parsing}

We also report initial dependency parsing results.
We implemented our algorithms within the TurboParser \cite{Martins:13}. That is, every other aspect 
of the parser: feature set, probabilistic pruning algorithm, inference algorithm etc., is kept fixed but
training is performed with SWVP. We compare our results to the parser performance 
with CSP training (which comes with the standard implementation of the parser).

%we added our implementation to replace the parameter vector update that is implemented in the 
%function \textit{MakeGradientStep}.

We experiment with the datasets of the CoNLL 2007 shared task on multilingual dependency 
parsing \cite{Conll:07}, for a total of 9 languages. 
We followed the standard train/test split of these dataset. For SWVP, we randomly sampled 1000 sentences from 
each training set to serve as development sets and tuned the parameters as in the synthetic data experiments. 
CSP is trained on the training set and applied to the test set without any development set involved. 
%That is, the total training and development data used by both algorithm is identical.
We report the Unlabeled Attachment Score (UAS) for each language and model.

% \isubsection{Named Entity Recognition}{sec:NER}

%--------------------------------------------------------------------------%
\isection{Results}{sec:results}
%--------------------------------------------------------------------------%

\begin{table*}[t!]
\begin{center}
\begin{scriptsize}
%\begin{tabular}{|c|c|p{0.1\textwidth}|c|c|p{0.05\textwidth}|c|c|p{0.05\textwidth}|c|c|} 
\begin{tabular}{|c|c|c|c|c|c|c|c|c|c|c|} 
\hline
 \multirow{2}{*}{} & \multicolumn{3}{|c|}{simple(++), learnable(+++)}  & \multicolumn{3}{|c|}{simple(++), learnable(++)} & \multicolumn{3}{|c|}{simple(+), learnable(+)}  \\ 
\cline{2-10}
% & average & std & \# won CSP & average & std & \# won CSP & average & std & \# won CSP & average & std & \# won CSP  \\ 
Model& Acc. (std) & \# Wins & Gener. & Acc. (std) & \# Wins & Gener. & Acc. (std) & \# Wins &  Gener. \\ 
\hline

B-WM & 75.47(3.05) & {\bf 9/10} & 10/10 &  {\bf 63.18 (1.32)} & {\bf 9/10} & 10/10 & 28.48 (1.9)&  5/10 & 10/10 \\ 
\hline
B-WMR & {\bf 75.96 (2.42)} & 8/10 & 10/10 & 63.02 (2.49) & 9/10 & 10/10 & 24.31 (5.2)& 4/10 &  10/10\\ 
\hline
\hline
A-WM & 74.18 (2.16) & 7/10 & 10/10 & 61.65 (2.30) & {\bf 9/10} & 10/10 & {\bf 30.45 (1.0)}& {\bf 6/10} & 10/10  \\	
\hline
A-WMR & 75.17 (3.07) & 7/10 & 10/10 & 61.02 (1.93) & 8/10 & 10/10 &  25.8 (3.18)& 2/10 &  10/10 \\ 
\hline 
\hline
CSP & 72.24 (3.45) & NA & NA & 57.89 (2.85) &  NA & NA & 25.27(8.55) & NA &  NA \\ 
\hline 
\end{tabular}
\caption{Overall Synthetic Data Results. 
\textit{A-} and \textit{B-} denote an aggressive and a balanced approaches, respectively.
Acc. (std) is the average and the standard deviation of the accuracy across 10 test sets. 
\# Wins is the number of test sets on which the SWVP algorithm outperforms CSP. 
Gener. is the number of times the best $\beta$ hyper-parameter value on the development set 
is also the best value on the test set, or the test set accuracy with the best development 
set $\beta$ is at most 0.5\% lower than that with the best test set $\beta$.}
\label{table:SynthaticRes}
\end{scriptsize}
\end{center}
\vspace{-0.2cm}
\end{table*}

\begin{table*}[t!]
\begin{center}
\begin{scriptsize}
\begin{tabular}{|c|c||c|c|c||c||c|c|c|} 
\hline 
\multirow{1}{*}{} & \multicolumn{4}{|c|}{First Order}  & \multicolumn{4}{|c|}{Second Order} \\ 
\cline{2-9}
Language & CSP & B-WM & Top B-WM & Test B-WM & CSP & B-WM & Top B-WM & Test B-WM\\ 
\hline
English & 86.34 & 86.4 & \textbf{86.7}  & 86.7 & {\bf 88.02} & 87.82  & 87.82 & 87.92 \\ 
\hline
Chinese & 84.60 & 84.5 & \textbf{85.04}  & 85.05 & 86.82 & 86.69  & \textbf{86.83} & 87.02 \\ 
\hline 
Arabic & 79.09 & 79.17  & \textbf{79.21}  & 79.21 & 76.07 & 75.94  & \textbf{76.09} &  76.09\\ 
\hline 
Greek & \textbf{80.41} & 80.20  & 80.28  & 80.28 & 80.31 & \textbf{80.40}  & \textbf{80.40} & 80.61 \\ 
\hline 
Italian & 84.63 & 84.64  & \textbf{84.74} &  84.70 & 84.03 & 84.08  & \textbf{84.15} & 84.28 \\ 
\hline 
Turkish & \textbf{83.05} & 82.89  & 82.89 &  82.89 & 83.02 & \textbf{83.04}  & \textbf{83.04} &  83.31 \\ 
\hline 
Basque  & 79.47 & \textbf{79.54}  & \textbf{79.54} & 79.54 & 80.52 & 80.57  & \textbf{80.63} & 80.64\\ 
\hline 
Catalan & \textbf{88.51} & 88.46  & 88.50 & 88.5 & 88.71 & \textbf{88.81}  & \textbf{88.81} & 88.82\\ 
\hline 
Hungarian & \textbf{80.17} & 80.07  & 80.07 & 80.21 & \textbf{80.61} & 80.45  & 80.45 & 80.55\\ 
\hline 
\hline
Average & 83.69 & 83.65 & \textbf{83.77} & 83.79 & 83.12 & 83.08 & \textbf{83.13} & 83.35\\ 
\hline 
\end{tabular}
\caption{First and second order dependency parsing UAS results for CSP trained models, 
as well as for models trained with SWVP with a balanced $\gamma$ selection (B) and with a weighted margin (WM) strategy. 
For explanation of the B-WM, Top B-WM, and Test B-WM see text. For each 
language and parsing order we highlight the best result in bold font, but this do not include results from Test B-WM 
as it is provided only as an upper bound on the performance of SWVP. 
}
\label{table:DepRes12}
\end{scriptsize}
\end{center}
\vspace{-0.2cm}
\end{table*}

\paragraph{Synthetic Data}
Table \ref{table:SynthaticRes} presents our results. In all three setups an SWVP algorithm is superior. 
Averaged accuracy differences between the best performing algorithms and CSP 
are: 3.72 (B-WMR, (simple(++), learnable(+++))), 5.29 (B-WM, (simple(++), learnable(++))) 
and 5.18 (A-WM, (simple(+), learnable(+))).
In all setups SWVP outperforms CSP in terms of averaged performance 
(except from B-WMR for (simple(+), learnable(+))). Moreover, the weighted models are more stable than 
CSP, as indicated by the lower standard deviation of their accuracy scores. Finally, for the more 
simple and learnable datasets the SWVP models outperform CSP in the majority of cases (7-10/10).

We measure generalization from development to test data in two ways. 
First, for each SWVP algorithm we count the number of times its $\beta$ parameter 
results in an algorithm that outperforms the CSP on the development set but not on 
the test set (not shown in the table). Of the 120 comparisons reported in the table (4 SWVP models, 
3 setups, 10 comparisons per model/setup combination) this happened once (A-MV, (simple(++), learnable(+++)).

Second, we count the number of times the best development set value of the $\beta$ hyper-parameter is also the best value on the test set, or the test set accuracy with the best development set $\beta$ is at most 0.5\% lower than that with the best test set $\beta$.
The \textit{Generalization} column of the table shows that this has not happened in all of the 120 runs of SWVP.

\paragraph{Dependency Parsing} 

Results are given in Table 2.
For the SWVP trained models we report three numbers: (a) B-WM is the standard setup where the $\beta$ hyper parameter 
is tuned on the development data; 
(b) For Top B-WM we first selected the models with a UAS score within 0.1\% of the best development data result, 
and of these we report the UAS of the model that performs best on the test set;
and (c) Test B-WM reports results when $\beta$ is tuned on the test set. This measure provides an upper bound 
on SWVP with our simplistic $JJ$ (\secref{sec:swvp-variants}).
%($JJ = \{JJ_{x} = \{\{1\}, \{2\}, \ldots \{L_{x}\}\} | (x,y) \in D\}$).

Our results indicate the potential of SWVP. Despite our simple $JJ$ set, Top B-WM and Test B-WM 
improve over CSP in 5/9 and 6/9 cases in first order parsing, respectively, and in 
7/9 cases in second order parsing. In the latter case, Test B-WM improves the UAS over CSP in 0.22\%  
on average across languages. Unfortunately, SWVP still does not generalize well from train to test data  
as indicated, e.g., by the modest improvements B-WM achieves over CSP in only 5 of 9 languages in 
second order parsing. 

%--------------------------------------------------------------------------%
\isection{Conclusions}{sec:conclusions}
%--------------------------------------------------------------------------%

We presented the Structured Weighted Violations Perceptron (SWVP) algorithm, a generalization of the 
Structured Perceptron (CSP) algorithm that explicitly exploits the internal structure of 
the predicted label in its update rule. We proved the convergence of the algorithm for linearly separable 
training sets under certain conditions on its parameters, and provided generalization and mistake 
bounds.

In experiments we explored only very simple configurations of the SWVP parameters - $\gamma$ and $JJ$. 
Nevertheless, several of our SWVP variants outperformed the CSP special case in synthetic data experiments. 
In dependency parsing experiments, SWVP demonstrated some improvements over CSP, 
but these do not generalize well. While we find these results somewhat encouraging, 
they emphasize the need to explore the much more flexible $\gamma$ and $JJ$ selection strategies 
allowed by SWVP (Sec. 4.2). In future work we will hence develop $\gamma$ and $JJ$ selection algorithms, 
where selection is ideally performed jointly with inference (property 2, Sec. 4.2), 
to make SWVP practically useful in NLP applications.

%--------------------------------------------------------------------------%

% \section*{Appendices}
\appendix
%\begin{appendices}

\section{Proof Observation 1.}
\label{appendix:Observation1}
{\bf Observation 1.} For linearly separable data $D$ and a set $JJ$, every unit vector \textbf{u} that separates 
the data with margin $\delta$, also separates the data with respect to mixed assignments with $JJ$, 
with margin $\delta^{JJ} \ge \delta$. Likewise, it holds that $R^{JJ} \leq R$.

\begin{proof}
For every training example $(x,y) \in D$, it holds that: 
$\cup_{z \in \mathcal{Y}(x)} m^J(z,y) \subseteq \mathcal{Y}(x)$. 
As \textbf{u} separates the data with margin $\delta$, it holds that:
\begin{small}
\begin{align*}
\textbf{u}\cdot\Delta\phi(x,y,m^J(z,y)) \ge \delta^{JJ_{x}} ,& \quad \forall z \in \mathcal{Y}(x), J \in JJ_{x}.\\
\textbf{u}\cdot\Delta\phi(x,y,z) \ge \delta,& \quad \forall z \in \mathcal{Y}(x).
\end{align*}
\end{small}
Therefore also $\delta^{JJ_{x}} \ge \delta$. As the last inequality 
holds for every $(x,y) \in D$ we get that $\delta^{JJ} = \min_{(x,y) \in D} \delta^{JJ_{x}} \ge \delta$.\\
From the same considerations it holds that $R^{JJ} \leq R$. This is because 
$R^{JJ}$ is the radius of a subset of the dataset with radius $R$ (proper subset if $\exists (x,y) \in D, [L_{x}] \notin JJ_{x}$, 
non-proper subset otherwise).
\end{proof}

\isection{Mistake and Generalization Bounds - Non Separable Case}{sec:bounds}
\label{appendix:mistake}
\paragraph{Mistake Bound}
Here we provide a mistake bound for the algorithm in the non-separable case. We start with the following definition and observation: 
\begin{definition}
\label{def:general}
Given an example $(x^i,y^i)\in D$, for a $\textbf{u},\delta$ pair define:
\begin{align*}
r^i&=\textbf{u}\cdot\phi(x^i,y^i)-\max\limits_{z\in \mathcal{Y}(x^i)}\textbf{u}\cdot\phi(x^i,z)\\
\epsilon_i &= \max\{0,\delta-r^i\}\\
{r^i}^{JJ} &=\textbf{u}\cdot\phi(x^i,y^i)- \\
&\max\limits_{z\in \mathcal{Y}(x^i),J \in JJ_{x^i}}\textbf{u}\cdot\phi(x^i,m^J(z,y^i))
\end{align*} 
Finally define:
$D_{\textbf{u},\delta}=\sqrt{\sum\limits_{i=1}^{n}\epsilon_i^2}$ 
\end{definition}
{\bf Observation 2.} For all $i$: $r^i \leq {r^i}^{JJ}$.\\
Observation 2 easily follows from Definition~\ref{def:general}. Following this observation we denote:
$r^{diff} = \min_{i}  \{{r^i}^{JJ} - r^i\} \geq 0$ and present the next theorem:
\begin{theorem}
\label{mistake_bound}
For any training sequence $D$, for the \textbf{first} pass over the training set of 
the CSP and the SWVP algorithms respectively, it holds that:
\begin{align*}
\#mistakes-CSP &\le \min\limits_{\textbf{u}:\|\textbf{u}\|=1,\delta>0} \frac{(R+D_{\textbf{u},\delta})^2}{\delta^2}.\\
\#mistakes-SWVP&\le\min\limits_{\textbf{u}:\|\textbf{u}\|=1,\delta>0} \frac{(R^{JJ}+D_{\textbf{u},\delta})^2}{(\delta  + r^{diff})^2}.
\end{align*}     
\end{theorem}
As $R^{JJ} \leq R$ (Observation 1) and $r^{diff} \geq 0$, we get a tighter bound for SWVP. The proof for \#mistakes-CSP is given at \cite{CollinsPerceptron}. The proof for \#mistakes-SWVP is given below.
\begin{proof}
We transform the representation $\phi(x,y)\in \mathbb{R}^d$ into 
a new representation $\psi(x,y)\in \mathbb{R}^{d+n}$ as follows: 
for $i=1,...,d: \psi_i(x,y)=\phi_i(x,y)$, for $j=1,...,n: \psi_{d+j}(x,y)=\Delta$ if $(x,y)=(x^j,y^j)$ 
and $0$ otherwise, where $\Delta>0$ is a parameter.\\
Given a $\textbf{u},\delta$ pair define $\textbf{v}\in \mathbb{R}^{d+n}$ as follows: 
for $i=1,...,d: \textbf{v}_i=\textbf{u}_i$, 
for $j=1,...,n: \textbf{v}_{d+j}=\frac{\epsilon_j}{\Delta}$.\\
Under these definitions we have:
\[\textbf{v}\cdot\psi(x^i,y^i)-\textbf{v}\cdot\psi(x^i,z) \ge  \delta, \quad \forall i,z\in \mathcal{Y}(x^i).\]
For every $i,z\in \mathcal{Y}(x^i), J \in JJ_{x^i}:$
\[\textbf{v}\cdot\psi(x^i,y^i)-\textbf{v}\cdot\psi(x^i,m^J(z,y^i))\ge \delta + r^{diff}.\]
\[\|\psi(x^i,y^i)-\psi(x^i,m^J(z,y^i))\|^2 \le ({R^{JJ}})^2+\Delta^2.\]
Last, we have,
\[\|\textbf{v}\|^2=\|\textbf{u}\|^2+\sum\limits_{i=1}^{n}\frac{\epsilon_i^2}{\Delta^2} = 1+\frac{D_{\textbf{u},\delta}^2}{\Delta^2}.\]

We get that the vector $\frac{\textbf{v}}{\|\textbf{v}\|}$ linearly separates the data with respect to single decision assignments 
with margin $\frac{\delta}{\sqrt{1+\frac{D_{U,\delta}^2}{\Delta^2}}}$.
Likewise, $\frac{\textbf{v}}{\|\textbf{v}\|}$ linearly separates the data with respect to mixed assignments with $JJ$,
with margin $\frac{\delta + r^{diff}}{\sqrt{1+\frac{D_{\textbf{u},\delta}}{\Delta^2}}}$.
Notice that the \textbf{first} pass of SWVP with representation $\Psi$ is identical to the first pass with representation $\Phi$ because the parameter weight for the additional features affects only a single example of the training data and do not affect the classification of test examples. By theorem 1 this means that the \textbf{first} pass of SWVP with 
representation $\Psi$ makes at most $\frac{(({R^{JJ}})^2+\Delta^2)}{(\delta + r^{diff})^2}\cdot\big( 1+\frac{D_{\textbf{u},\delta}^2}{\Delta^2} \big)$.\\
We minimize this w.r.t $\Delta$, which gives: $\Delta = \sqrt{R^{JJ}D_{\textbf{u},\delta}}$, and 
obtain the result guaranteed in the theorem.
\end{proof}

We have bounded the number of mistakes SWVP is making in an on-line setup. 
We next provide guarantees as to how well the algorithm generalizes to a new example.

\paragraph{Generalization Bound}
Let us consider the training set $D$ as an ordered sequence: $D = \{(x^1,y^1),\ldots,(x^n,y^n)\}$, and let us run the SWVP online algorithm on this sequence. At each round $t=1,\ldots,n$, the algorithm may update the weight vector $\textbf{w}$, so we get a sequence of weight vectors $\textbf{w}^1,\ldots,\textbf{w}^n$, from which we can create an hypotheses sequence of the form $h^t(x)=\arg\max_{y' \in \mathcal{Y}(x)} \textbf{w}^t \cdot \phi(x,y')$. 

To check the algorithm success in generalizing to a new test example $(x^{n+1},y^{n+1})$, we need to decide which hypothesis to use from the above sequence, under the assumption that both the training examples and the new test example are drawn i.i.d from an (unknown) distribution $P(x,y)$.

\newcite{Freund:9l} presented the voted perceptron, a batch variant of the perceptron algorithm, and \cite{CollinsPerceptron} presented an approximation for this variant called the averaged parameters perceptron that holds the same generalization guarantees. 
We adapt the averaged parameters setting to our algorithm. The resulting adaptation of \cite{Freund:9l} then states:
\begin{theorem}[Freund \& Schapire 99]
Assume all examples are generated i.i.d. at random. Let $(x^1,y^1),\ldots,(x^n,y^n)$ be a sequence of training examples and let $(x^{n+1},y^{n+1})$ be a test example. 
For a pair $\textbf{u}$, $\delta$ such that $\|\textbf{u}\|=1$ and $\delta>0$ 
define $D_{\textbf{u},\delta}$ as before. Then the probability (over the choice of $n+1$ examples) that the voted SWVP algorithm does not predict $y^{n+1}$ on test instance $x^{n+1}$ is at most
\[\frac{2}{n+1}\mathbb{E}_{n+1} \Bigg( \inf\limits_{\textbf{u}:\|\textbf{u}\|=1,\delta>0} \frac{(R^{JJ}+D_{U,\delta})^2}{(\delta  + r^{diff})^2} \Bigg)\]
where $\mathbb{E}_{n+1}$ is an expected value taken over $n+1$ examples.
\end{theorem} 

Note that the adaptation of \cite{Freund:9l} to the original CSP algorithm provided by \cite{CollinsPerceptron} gives the generalization bound of 
$\frac{2}{n+1}\mathbb{E}_{n+1} \left( \inf\limits_{\textbf{u}:\|\textbf{u}\|=1,\delta>0} \frac{(R+D_{U,\delta})^2}{\delta^2}\right).$ 
This means that the generalization bound of SWVP is upper bounded 
by the generalization bound of CSP (convergence property 1 and theorem 2).

%\end{appendices}
\section*{Acknowledgments}
The second author was partly supported by a research grant from the GIF Young Scientists' Program (No. I-2388-407.6/2015): Syntactic Parsing in Context.
\bibliography{emnlp2016}
\bibliographystyle{emnlp2016}

\end{document}